\documentclass[10pt,twocolumn,letterpaper]{article}

\usepackage{cvpr}
\usepackage{times}
\usepackage{epsfig}
\usepackage{graphicx}
\graphicspath{{./}}
\DeclareGraphicsExtensions{.pdf}
\usepackage{amsmath}
\usepackage{amssymb}
\usepackage{bm}
\usepackage{algorithm}
\usepackage{algorithmic}
\usepackage{array}
\usepackage{makecell}
\usepackage{placeins}
\usepackage{stfloats}
\usepackage{multirow}

\newenvironment{packed_itemize}{
\vspace{-0.15cm}\begin{itemize}
  \setlength{\itemsep}{1pt}
  \setlength{\parskip}{0pt}
  \setlength{\parsep}{0pt}
}{\end{itemize}}

\def\y{\mathbf{y}}
\def\x{\mathbf{x}}

\def\U{\mathbf{U}}

\def\z{\mathbf{z}}

\def\x{\mathbf{x}}
\def\m{\mathbf{m}}

\newcommand{\Ff}{\mathfrak{F}}

\usepackage{macro}

\usepackage[breaklinks=true,colorlinks,bookmarks=false]{hyperref}

\cvprfinalcopy 

\def\cvprPaperID{1854} 

\ifcvprfinal\fi
\begin{document}

\title{CNNs are Globally Optimal Given Multi-Layer Support}

\author{Chen Huang \quad \quad Chen Kong \quad \quad Simon Lucey\\
Robotics Institute, Carnegie Mellon University\\
{\tt\small \{chenh2,chenk,slucey\}@andrew.cmu.edu}
}

\maketitle

\begin{abstract}
Stochastic Gradient Descent (SGD) is the central workhorse for training modern CNNs. Although giving impressive empirical performance it can be slow to converge. In this paper we explore a novel strategy for training a CNN using an alternation strategy that offers substantial speedups during training. We make the following contributions: (i) replace the ReLU non-linearity within a CNN with positive hard-thresholding, (ii) re-interpret this non-linearity as a binary state vector making the entire CNN linear if the multi-layer support is known, and (iii) demonstrate that under certain conditions a global optima to the CNN can be found through local descent. We then employ a novel alternation strategy
(between weights and support) for CNN training that leads to
substantially faster convergence rates, nice theoretical properties,
and achieving state of the art results across large scale datasets
(e.g. ImageNet) as well as other standard benchmarks. 
\end{abstract}

\section{Introduction}
Most modern Convolutional Neural Networks (CNNs) can be expressed
simply in terms of the composition of an affine transform followed
by a non-linear function\footnote{Recent work by~\cite{SpringenbergDBR14}
  has demonstrated that state of the art CNN performance can be
  attained without max pooling or other non-linearities, resulting in just the use of linear
  affine operations and ReLU within the network.} such
as~$\mbox{ReLU}(\x) = \max(\x,0)$. The output for each layer~$l$ in the network can be
expressed recursively as, 
\begin{equation}
\z^{(l)} = \mbox{ReLU} \{ \Wv^{(l)} \z^{(l-1)} - \betav^{(l)} \} 
\label{Eq:introd}
\end{equation}
where~$\{ \Wv^{(l)}, \betav^{(l)} \}$ are the affine parameters and
~$\z^{(l-1)}$ is the previous layer's output
and~$\z^{(0)} = \x$ being the input
signal. It is well understood that when one applies ReLU to a vector (as in Eq.~\ref{Eq:introd} ) the resulting
output~$\z^{(l)}$ will be sparse such that~$\m^{(l)} \odot \z^{(l)} =
\z^{(l)}$ where~$\odot$ represents the Hadamard product operator. We
shall refer to~$\m^{(l)}$ herein as the support of~$\z^{(l)}$ where the individual
elements are constrained to the binary values~$\{0,1\}$.  Further, we shall refer to the
concatenation of these supports across all~$L$ layers within the CNN
as the multi-layer support.

\begin{table}[t]
\caption{Image classification results of our alternating training algorithm and existing local descent algorithm applied to different network architectures. Our algorithm achieves state-of-the-art results at faster convergence rate, with nice theoretical properties and many applications in deep learning.}
\centering
\resizebox{0.9\linewidth}{!}{
\begin{tabular}{!{\vrule width1pt} c|c|c|c !{\vrule width1pt}}
    \Xhline{1pt}
     Error (\%) & Architecture & Published & Ours \\ \Xhline{1pt}
     MNIST      & \cite{Lin14} & \textbf{0.47} & \textbf{0.47} \\
     CIFAR-10 (w/o aug)  & \cite{SpringenbergDBR14} & \textbf{9.08} & 9.12 \\
     CIFAR-10 (w/ aug) & \cite{SpringenbergDBR14} & 7.25 & \textbf{7.20} \\ 
     CIFAR-100   & \cite{SpringenbergDBR14} & 33.71 & \textbf{33.68} \\ 
     SVHN   & \cite{Lin14} & \textbf{2.35} & 2.36 \\ \Xhline{1pt}
    \end{tabular}
}
\label{Tab:tb1}
\end{table}

\begin{figure}[t]
\begin{center}
\includegraphics[width=0.75\linewidth]{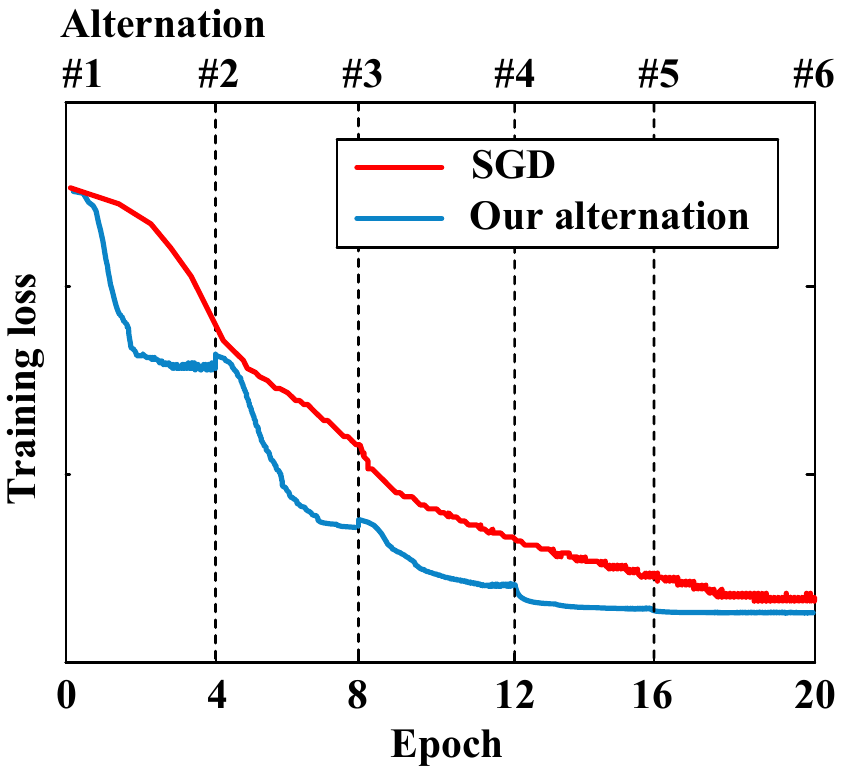}
\end{center}
\caption{Comparison of the training curves of SGD and our alternation strategy on MNIST. The alternation iteration is identified at the corresponding epoch number. Our alternating training converges faster than SGD but with similar performance.}
\label{fig1}
\end{figure}

In this paper we propose a re-interpretation of the 
CNN architecture in Eq.~\ref{Eq:introd} such that the network can now be expressed as, 
\begin{equation}
\z^{(l)} = \mbox{diag}(\m^{(l)}) (\Wv^{(l)} \z^{(l-1)} - \betav^{(l)} )
\label{Eq:ours}
\end{equation}
where the ReLU non-linearity at each layer is instead replaced by a binary
support. Interestingly, the re-interpretation in Eq.~\ref{Eq:ours} 
becomes completely linear if one knows the multi-layer
support for the input signal a priori. Obviously, in practice one cannot determine this multi-layer support without the
multi-layer affine weights - so it seems nothing is gained as this modified network remains just
as non-linear and problematic to optimize as before.

\noindent \textbf{A Case for Alternation:}
Modern CNNs are typically optimized through stochastic gradient
descent (SGD). More efficient
optimization strategies are possible such as quasi-Netwon
methods~\cite{goodfellow16}, however, they are problematic due to their high computational cost and the poor
positive curvature of CNN objective functions in general. Alternation can be intepreted as a kind of block coordinate
descent where the global solution to one factor is found while keeping
the other factors fixed~\cite{damped_newton05}. Although alternation is not
guaranteed of reaching a local minima, the objective will be reduced
after every iteration. As argued by~\cite{damped_newton05}, the convergence of alternation
iterations is initially very good, and has typically low computational
overhead - making it an attractive iteration scheme for many large
scale problems in vision and learning. 

As discussed by~\cite{goodfellow16}, however, alternation is typically
dismissed as a viable optimization strategy for modern CNNs. Strategies like alternation
(and block coordinate descent) make the most sense when factors in the
objective can be separated into groups that play relatively isolated
roles. Such a separation of parameters is problematic when applied to CNNs as
nearly every parameter in the network heavily influences the other. In
this paper we argue that by re-interpreting the ReLU non-linearity as
a binary support parameter - one can
now orchestrate a parameter separation
that makes an alternation strategy attractive over conventional
SGD. Specifically, we advocate for a strategy where we hold the
multi-layer support fixed, and then solve for the multi-layer affine
weights until convergence. We then estimate the support using the updated weights -
iterating the whole alternation process until a good solution the CNN is found. 

A necessary component for our proposed alternation strategy to be
feasible is that an approximate global solution to our CNN's weights can be
determined given the known multi-layer support for each training
example. This may seem at first glance a tall order, as it has been
well documented~\cite{Haeffele_2017_CVPR} that modern CNNs
with some minor exceptions have no guarantees of converging to
a global solution. This brings us to the central insight of our paper
- as well as inspiring our paper's title. Specifically, we demonstrate that we can find the global
solution (through local descent) to the weights of our proposed CNN
architecture assuming known multi-layer support for all the training
examples. This result in isolation is of little practical interest as
one never knows the multi-layer support for an
input signal a priori. However, the result is of considerable interest
when viewed from its applicability for training CNNs within an
efficient alternation framework.

\noindent \textbf{Contributions:}
We make the following contributions:-
\begin{packed_itemize}

\item Demonstrate how one can linearly separate weights from
  support within a modern CNN using a modified ReLU non-linearity. Based on this separation we can then
  demonstrate how the problem of multi-layer affine weight estimation
  can be represented as determining a rank one tensor. Further, even
  though this problem is clearly non-convex we
  characterize under which conditions that a global optima can be found
  using local descent assuming the multi-layer support of the training examples is known (Section~\ref{subsec:weight_opt}).

\item Present empirical results which show the utility of our proposed alternation strategy for CNN
  optimization. Specifically, we demonstrate substantially faster
  convergence results (over traditional SGD) over a number of
  benchmarks while achieving (and in some circumstances surpassing)
  state of the art in actual recognition performance (Section~\ref{subsec:benchmark}).

\item Finally, we discuss computational issues with our approach when attempting
  to learn a CNN on a large scale data (e.g. ImageNet) - as naively
  one would be expected to estimate the multi-layer support for every
  training example. Specifically, we advocate for a novel modification of
  our alternation strategy where a fixed multi-layer support is
  assumed across a mini-batch of the training examples. We demonstrate
  that this strategy can obtain state of the art performance, while
  still enjoying the rapid convergence properties of our original
  naive approach (Section~\ref{subsec:imagenet}).   

\end{packed_itemize}

\section{Related Work}
Deep neural networks have been revolutionary in many fields of machine learning and artificial intelligence, including computer vision, speed
recognition, and natural language processing. Their performance continues to improve with ever deepening network architectures~\cite{he2016deep}.

\noindent
{\bf Optimization:} Modern CNNs still remain difficult to train due to the nonconvexity of the optimization problem. They are typically optimized through SGD which often cannot return a global minima. More efficient optimization strategies exist such as quasi-Netwon methods~\cite{goodfellow16}, which however, incur high computational cost. One recent work in~\cite{papyan17} views the forward pass of a CNN as a thresholding pursuit serving a Multi-Layered Convolutional Sparse Coding (ML-CSC) model. The authors argued that improved results should be attainable by employing a more sophisticated pursuit algorithm. Promising results were attained, however, due to the computational complexity of the proposed pursuit strategy results were limited to small image datasets.


\noindent
{\bf Optimality:} Despite the empirical success of CNNs, the
theoretical underpinnings for this success and their optimality remains elusive. Of
particular note in this regard is the recent work of
Kawaguchi~\cite{kawaguchi2016deep} who studied the optimality of
simplistic deep {\em linear} neural networks. Kawaguchi proved that
any local minimum point is also a global minimum, and that any other
critical point is a saddle point for deep linear
networks. Further,~\cite{NguyenH17,Chulhee17} provided some conditions
for a critical point of the empirical risk function to be a global
minimum for both {\em linear} and {\em nonlinear} networks. Recently,
Haeffele and Vidal~\cite{Haeffele_2017_CVPR} derived sufficient
conditions to guarantee that local minima for their proposed network are globally optimal, requiring both the network output and the regularization to be positively homogeneous, with the regularization being designed to control the network size. However, the results only apply to networks with one hidden layer, or with multiple subnetworks connected in parallel. Our work does not have such restrictive assumptions about networks.

\section{Separating Weights from Supports}
Most modern Convolutional Neural Networks (CNNs) are built with a collection of feature layers followed by a small number of fully connected layers.
Throughout most of the theoretical treatment of this paper we follow~\cite{SpringenbergDBR14} to replace max-pooling by a convolutional layer with increased strides.
This makes the convolutional pooling layer linear, leaving activation layers as the sole nonlinearity in a network.
We can express such a network simply in terms of the concatenation of an affine transform~$(\Wv\xv)$\footnote{A layer within CNNs is said to be convolutional if $\mathbf{W}$ has a banded strided Toeplitz structure, otherwise it is considered fully connected.} followed by a non-linear function~$\eta\{\Wv\xv; \betav \}$ such as~$\mbox{ReLU}(\Wv\xv - \betav)$.
For example a three layer network can be expressed as,
\begin{equation}
\mathcal{F}(\mathbf{x}; \Wv, \betav) = \Wv^{(3)} \eta \{ \Wv^{(2)} \eta \{\Wv^{(1)} \xv ; \betav^{(1)}\};\betav^{(2)} \} - \betav^{(3)}.
\label{Eq:orig}
\end{equation}
The above three layer network can be generalized to~$L$ layers by expressing the~$\ell-$th network output recursively as,
\begin{equation}
\mathbf{z}^{(\ell)} = \eta \{ \Wv^{(\ell)} \zv^{(\ell-1)} ; \betav^{(\ell)} \}
\label{Eq:hidden}
\end{equation}
where~$\zv^{(\ell-1)}$ is the previous layer's output and~$\zv^{(0)} = \xv$. We can therefore define the~$L$ layer network function as,
\begin{equation}
\Fc(\xv; \Wv,\betav) = \Wv^{(L)} \zv^{(L-1)} - \betav^{(L)}
\label{Eq:layerL}
\end{equation}

Such a network function can be trained discriminatively on the following objective function,
\begin{equation}
\arg \min_{\mathbf{W},\betav} \sum_{n=1}^{N}
\mathcal{L}\{\mathbf{y}_{n}, \mathcal{F}(\mathbf{x}_{n}; \mathbf{W}, \betav) \}
\label{Eq:obj}
\end{equation}
where~$N$ defines the number of training examples,~$\mathbf{y}_{n}$ and~$\mathbf{x}_{n}$ are the~$n$-th output label vector and input vector respectively.
Numerous loss functions~$\mathcal{L}$ can be employed depending on the nature of task ranging from
cross-entropy~\cite{deng2006cross} to least-squares error.

\begin{figure}[t]
\begin{center}
\includegraphics[width=1.0\linewidth]{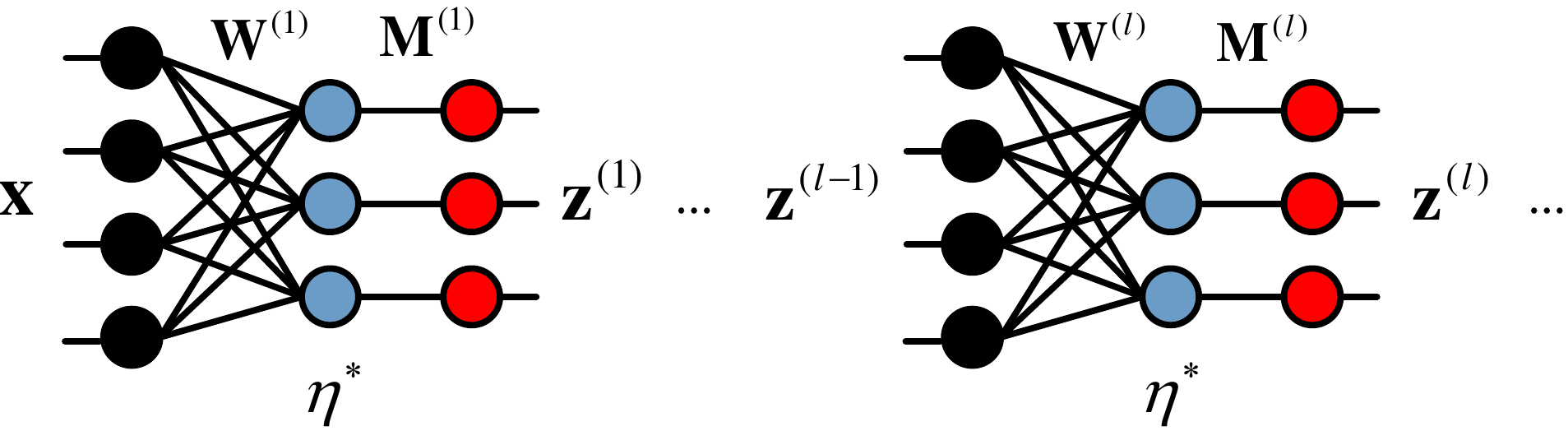}
\end{center}
\caption{CNN reinterpreted as weights (produce real-valued units, blue) and support (produce binary masks, red, via positive hard-thresholding $\eta^{*}$). This naturally calls for an alternation strategy between weights and support for CNN training.}
\label{fig2}
\end{figure}

\subsection{Positive Hard-Thresholding as a Non-linearity}
As previously discussed, a popular non-linearity within CNN literature at the moment is the
Rectified Liner Unit (ReLU) defined as,
\begin{eqnarray}
\eta\{ q ; \beta\} & = & \max(q - \beta, 0) \\
& = & \bigg \{ \begin{tabular}{cc} $q - \beta$, & if $q > \beta$ \\
0, & otherwise \end{tabular} \nonumber \;\;,
\end{eqnarray}
where scalar $q$ is the input and $\beta$ is a learnable threshold/bias.
With ReLU applied independently across elements of the input vector, the output of the CNN at layer~$\ell$ in Equation~\ref{Eq:hidden} can be rewritten as
\begin{eqnarray}
    \zv^{(\ell)}  =  \eta\{ \Wv^{(\ell)}\zv^{(\ell-1)} ; \betav^{(\ell)} \}.
\end{eqnarray}

One drawback of ReLU is that the output mixes the input with the threshold/bias complicating the separation of weights from support. To facilitate this separation we advocate to instead use a positive hard thresholding operator to replace ReLU within a CNN. Positive hard-thresholding is defined as
\begin{equation}
    \eta^{*} \{q; \beta\} = \begin{cases}
        q & \text{if}~q > \beta\\
        0 & \text{otherwise}.
        \end{cases}
    \label{Eq:hard}
\end{equation}
Positive hard-thresholding not only facilitates the decoupling of the weights from support, but we also show in our experiments section that it has no discernable empirical drawbacks compared to ReLU and even improves performance in some circumstances.

There is also a theoretical motivation for this change in non-linearity. Papyan \etal~\cite{papyan17} recently demonstrated that the non-linearity ReLU can be viewed as the closed form solution to the following non-non-negative $\ell_1$ minimization:
\begin{equation}
    \eta\{ q ; \beta \} = \argmin_{p}~\Vert q - p\Vert_2^2 + 2\beta \Vert p\Vert_1, ~ \st p > 0.
    \label{Eq:sparse}
\end{equation}
Taking the view that the~$\ell_{1}$ norm term in Equation~\ref{Eq:sparse} is a convex approximation to the sparisty inducing~$\ell_{0}$ norm one could rewrite the objective as,   
\begin{equation}
    \eta^{*}\{ q ; \beta \} = \argmin_{p}~\Vert q - p\Vert_2^2 + 2\beta \Vert p\Vert_0, ~ \st p > 0.
\end{equation}
Papyan~\etal proposed to potentially use the closed form solution to this $\ell_0$ problem as an alternative non-linearity to ReLU within a CNN.
This closed form solution is the positive hard-thresholding non-linearity defined in Equation~\ref{Eq:hard}.

Armed now with the positive hard-thresholding non-linearity, we can now reinterpret a CNN in a new manner. Specifically for each layer of the network we can form the indicator binary support vector
\begin{equation}
    \mathbf{m}^{(\ell)} = [(\scalebox{0.8}{$q_1^{(\ell)} > \beta_1^{(\ell)}$}), \cdots, (\scalebox{0.8}{$q_M^{(\ell)} > \beta_M^{(\ell)}$})]^T
\end{equation}
to replace the positive hard-thresholding non-linearity. For example, one could rewrite the three layered network described earlier
in Equation~\ref{Eq:orig} as,
\begin{equation}
\mathcal{F}(\mathbf{x}; \mathbf{W},\betav) = \mathbf{W}^{(3)}
\Mv^{(2)} \Wv^{(2)} \Mv^{(1)}\Wv^{(1)} \xv - \betav^{(3)}.
\label{Eq:newthree}
\end{equation}
where~$\mathbf{M}^{(\ell)} = \diag(\mathbf{m}^{(\ell)})$.

\subsection{Decoupling Weights and Masks}
Let us consider the output of the first layer in a CNN $\zv^{(1)} = \Mv^{(1)}\Wv^{(1)}\xv.$
Vectorizing both sides gives,
\begin{equation}
    \zv^{(1)} = \Uv^{(1)}\vv^{(1)},
    \label{Eq:layer1}
\end{equation}
where $\Uv^{(1)} = \xv\otimes\Mv^{(1)}$, $\vv^{(1)} = \vec(\Wv^{(1)})$, and the operator $\otimes$ denotes a kronecker product.
This now shows a possibile strategy for  separating the weights from the binary support in the first layer.

Now, we demonstrate that this separation also holds for a multi-layer CNN.
Suppose the output of $\ell$-th layer holds the separation, that is $\zv^{(\ell)} = \Uv^{(\ell)}\vv^{(\ell)}.$
The output of the $(\ell+1)$-th layer is therefore,
\begin{eqnarray}
    \zv^{(\ell+1)} &= &\Mv^{(\ell+1)}\Wv^{(\ell+1)}\zv^{(\ell)}\\
                   &= &\Mv^{(\ell+1)}\Wv^{(\ell+1)}\Uv^{(\ell)}\vv^{(\ell)}. \nonumber
\end{eqnarray}
By applying vectorization sequentially one can derive,
\begin{equation}
    \zv^{(\ell+1)} = \Uv^{(\ell+1)} \vv^{(\ell+1)},
\end{equation}
where
\begin{eqnarray}
    \Uv^{(\ell+1)} = \vec\big(\Uv^{(\ell)}\big)\otimes \Mv^{(\ell+1)}\\
    \vv^{(\ell+1)} = \vv^{(\ell)}\otimes \vec\big(\Wv^{(\ell+1)}\big).
\end{eqnarray}

By mathematical induction, it is implied that a multi-layer CNN can always be separated into a bilinear multiplication of $\Uv$ and $\vv$, where $\Uv$ depends on the input examples and the binary support while $\vv$ only depends on the weights $\Wv$.
Formally, by setting $\Mv^{(L)} = \diag(\one)$ and~$\betav = \mathbf{0}$, the $L$ layer network function in Equation~\ref{Eq:layerL} can be rewritten as
\begin{equation}
    \Fc(\xv; \Wv, \Mv) = \Uv^{(L)} \vv^{(L)}.
\end{equation}
One can set $\Mv^{(L)} = \diag(\one)$ as the non-linearity is not applied to the last layer, additionally one can set~$\betav^{(L)} = \mathbf{0}$ as the bias can easily be ignored (for recognition tasks) or estimated independently (e.g. least-squares loss). 

From this perspective, the objective function in Equation~\ref{Eq:obj} can be expressed as a bilinear objective, specifically
\begin{equation}
    \sum_{n=1}^{N} \mathcal{L}\{\yv_n, \Uv_n^{(L)}\vv^{(L)} \}.
    \label{Eq:obj1}
\end{equation}

\subsection{Recovery of Weights Given Support}
\label{subsec:weight_opt}
Given the bilinear form of the objective function, we now characterize the optimality of recovering weights, assuming the masks are given from oracle.
To illustrate the basic idea, we first restrict ourselves to a two-layer neural network:
\begin{equation}
    \Fc^{(2)}(\xv_n; \Wv) = \Uv_n^{(2)} \vv^{(2)}
\end{equation}
using a least-squares loss function~$\mathcal{L}(\y,\x) = ||\y - \x||_{2}^{2}$. 
Suppose that we have a set of training samples $\{(\xv_n, \yv_n)\}_{n=1,..., N}$.
Then the problem of estimating $\Wv$ given $\Mv$ is formally defined as the problem (P1):
\begin{equation}
    (P1): \quad \min_{\Wv^{(1)}, \Wv^{(2)}} \sum_{n=1}^N \Vert \yv_n - \Fc^{(2)}(\xv_n;\Wv) \Vert_2^2.
\end{equation}
For the sake of brevity, we define:
\begin{equation}
    \yv = \begin{bmatrix} \yv_1 \\ \vdots \\ \yv_N \end{bmatrix}, \quad
    \Uv = \begin{bmatrix} \Uv^{(2)}_1 \\ \vdots \\ \Uv^{(2)}_N \end{bmatrix}.
\end{equation}
Therefore, the problem $(P1)$ can be simplified as 
\begin{equation}
    (P2): \quad \min_{\Wv^{(1)}, \Wv^{(2)}} \Vert \yv - \Uv \vv^{(2)} \Vert_2^2.
\end{equation}

Further, inspired by the work~\cite{bhojanapalli2016global}, we define a lifting function $\Ff$ as a linear operator satisfying
\begin{equation}
    \Ff(\pv\qv^T) = \pv \otimes \qv,\quad \text{for any vectors}~\pv, \qv.
\end{equation}

By imposing the lifting function to the problem $(P2)$, one can rewrite the objective function as
\begin{equation}
    \Big\Vert \yv -  \U\Ff\big(\vec(\Wv^{(1)})\vec(\Wv^{(2)})^T\big)\Big\Vert_2^2
\end{equation}

By defining $\Vv = \vec(\Wv^{(1)})\vec(\Wv^{(2)})^T$, we know $\rank(\Vv) = 1$.
Thus, the problem $(P2)$ can be equivalently rewritten to
\begin{equation}
    \begin{aligned}
        (P3): \quad &\min_{\Vv} \Vert \yv - \Uv~\Ff(\Vv) \Vert_2^2\\
                    &\st~\rank(\Vv) = 1.
    \end{aligned}
\end{equation}
If a solution to $(P3)$ can be found, we then can utilize Singular Value Decomposition to extract $\Wv^{(1)}, \Wv^{(2)}$ from $\Vv$:
\begin{equation}
    \vec(\Wv^{(1)}) = \sqrt{\sigma^*}\uv^{*}, \quad \vec(\Wv^{(2)}) = \sqrt{\sigma^*}\vv^{*}
\end{equation}
where $\Vv = \sigma^*\uv^*(\vv^*)^T$.

\begin{theorem}
    \label{thm:opt}
    The globally optimal solution to $\Vv$ in the problem $(P3)$ can be achieved by stochastic gradient descent with high possibility if there exists a positive number $0< \epsilon < 1/10$, such that the matrix $\Uv$ satisfies
    \begin{equation}
        (\frac{9}{10}+\epsilon)m \le \lambda_{min}^2 \le \lambda_{max}^2 \le (\frac{11}{10}-\epsilon) m,
        \label{eq: bounds}
    \end{equation}
    where $\lambda_{min}, \lambda_{max}$ are the smallest and largest eigenvalue of $\Uv$, and $m$ is the size of $\yv$.
\end{theorem}
\begin{proof}
    From the property of eigenvalues, it is implies that, under certain conditions, for any vector $\bv$, it is true that
    \begin{equation}
        \lambda_{min}^2 \Vert \bv \Vert_2^2 \le \Vert \Uv \bv \Vert_2^2 \le \lambda_{max}^2 \Vert \bv \Vert_2^2.
    \end{equation}
    By applying the bounds in Equation~\ref{eq: bounds}, it is implied that
    \begin{equation}
        (\frac{9}{10}+\epsilon)\Vert \bv \Vert_2^2 \le \frac{1}{m}\Vert \Uv \bv \Vert_2^2 \le (\frac{11}{10}-\epsilon) \Vert \bv \Vert_2^2.
    \end{equation}
    Since $\Ff(\Vv)$ is a vector, the following must holds:
    \begin{equation}
        (\frac{9}{10} +\epsilon) \Vert \Ff(\Vv) \Vert_2^2 \le \frac{1}{m}\Vert \Uv \Ff(\Vv) \Vert_2^2 \le (\frac{11}{10}-\epsilon) \Vert \Ff(\Vv) \Vert_2^2.
    \end{equation}
    Because one possible construction of $\Ff$ is
    \begin{equation}
        \Ff(\Vv) = \vec(\Vv^T),
    \end{equation}
    one can see $\Vert\Ff(\Vv)\Vert_2^2 = \Vert \Vv \Vert_F^2$.
    Then
    \begin{equation}
        (\frac{9}{10}+\epsilon) \Vert \Vv \Vert_F^2 \le \frac{1}{m}\Vert \Uv \Ff(\Vv) \Vert_2^2 \le (\frac{11}{10}-\epsilon) \Vert \Vv \Vert_F^2.
    \end{equation}
    By defining a linear operation $\Ac(\Vv) = \Uv\Ff(\Vv)$, we have
    \begin{equation}
        (\frac{9}{10}+\epsilon) \Vert \Vv \Vert_F^2 \le \frac{1}{m}\Vert \Ac(\Vv) \Vert_2^2 \le (\frac{11}{10} -\epsilon) \Vert \Vv \Vert_F^2.
    \end{equation}
    This condition characterizes that the operator $\Ac$ satisfies the RIP condition with $(2, \frac{1}{10}-\epsilon)$.
    By the Theorem 3.1 in~\cite{bhojanapalli2016global}, a global optimal solution $\Vv^*$ will be discovered by stochastic gradient descent with high possibility.
\end{proof}

Note that even though we show a condition where global optimality occurs when we are getting to have local optimality, we cannot guarantee this in polynomial time. This is also pointed by Vidal~\etal\cite{Haeffele_2017_CVPR}.
More generally, Theorem~\ref{thm:opt} shall be generalized to multi-layer CNNs analogously by using multi-dimensional tensors, which is omitted here due to the sake of brevity. After we recover high-quality weights with the supports fixed, we can then estimate supports using the updated weights - iterating the whole alternation process (like block coordinate descent) until a good solution the CNN is found. 

\section{Experimental Results}

We start with evaluating our alternation strategy on four standard image classification benchmarks: CIFAR-10~\cite{Krizhevsky09}, CIFAR-100~\cite{Krizhevsky09}, MNIST~\cite{Lecun98gradient} and SVHN~\cite{Netzer11}. Two network architectures are employed - the All Convolutional Network (ALL-CNN)~\cite{SpringenbergDBR14} where ReLU is the only nonlinearity, and Network in Network (NIN)~\cite{Lin14} with nonlinearities beyond ReLU (\ie,~max pooling). We derive a method to estimate the multi-layer support for every training example from the considered small datasets. The resulting alternating training algorithm is found to often converge to a better local minima, even if ReLU is not the sole nonlinearity. Our better results also come independent of network architectures and capacities, as well as different datasets. More importantly, we demonstrate much faster convergence properties of our algorithm. Lastly, we propose a batch-based modification of our alternation algorithm to scale up to the large ImageNet dataset~\cite{NIPS2012_4824}. We show that by sharing a fixed multi-layer support across a batch of examples, our alternation algorithm is computationally more friendly than the naive per-example approach. We experimented with several popular deep and highly nonlinear networks (\eg,~GoogLeNet~\cite{Szegedy15}), achieving state-of-the-art performance for all networks. We performed all experiments using the \texttt{Caffe} framework.

\subsection{Image Classification on Small Benchmarks}
\label{subsec:benchmark}

On small datasets, one simple way to implement our alternation algorithm is to estimate the support mask for every training instance during alternation. We call this approach as ``I-alternation'' where the instance-wise mask can be stored or even cached online for small data. The closed-form solution to the mask can be obtained via solving the reconstruction problem:
\begin{equation}
    \min_{\m}~\Vert \x - \m \odot \x \Vert_2^2 + \beta \Vert \m\Vert_0, ~ m_i \in \{0,1\},
\end{equation}
where the least-square error is used (the solution can generalize to the cross-entropy or other loss). Obviously this objective function is separable with respect to each element $m_i$ in the mask $\m$. Given $m_i$ as a binary mask, the objective function of each subproblem is
\begin{equation}
    (x_i-m_i x_i)^2 + \beta m_i = \begin{cases}
        \beta & \text{if} \; m_i=1\\
        x_i^2 & \text{if} \; m_i=0.
        \end{cases}
\end{equation}

Therefore, the optimal mask solution is:
\begin{equation}
    m_i^* = \begin{cases}
        1 & \text{if}~x_i^2 > \beta\\
        0 & \text{otherwise}.
        \end{cases}
\end{equation}

We first evaluate this naive approach on the small CIFAR-10 dataset. The ALL-CNN network~\cite{SpringenbergDBR14} is adopted, which only consists of convolutional layers with ReLU nonlinearities and softmax layer. Such a linearized architecture provides a perfect testbed to specially evaluate the capability of our ReLU-induced alternation algorithm. We follow the detailed training settings of~\cite{SpringenbergDBR14}, and similarly use the models A, B and C with increasing network capacities. For testing, as is common practice we apply ReLU to each testing example. To predict image classes, we follow~\cite{SpringenbergDBR14} to produce 10 outputs for different positions at the top $1\times 1$ convolutional layer, and simply average them over the whole image before computing softmax probabilities.

\begin{table}[t]
\caption{Baseline comparisons in terms of classification error and the number of network parameters on CIFAR-10 dataset.}
\centering
\resizebox{0.8\linewidth}{!}{
\begin{tabular}{!{\vrule width1pt} c !{\vrule width1pt} c|c !{\vrule width1pt}}
    \Xhline{1pt}
       Model & Error (\%) &  \# params \\ \Xhline{1pt}
       ALL-CNN-A~\cite{SpringenbergDBR14} & 10.30 & \multirow{3}{*}{1.28M} \\ 
       Our I-alternation & \textbf{10.26} & \\
       Our I-alternation+MaskT & 14.35 & \\\Xhline{1pt}
       ALL-CNN-B~\cite{SpringenbergDBR14} & \textbf{9.10} & \multirow{3}{*}{1.35M} \\ 
       Our I-alternation & 9.13 & \\
       Our I-alternation+MaskT & 13.25 & \\\Xhline{1pt}
       ALL-CNN-C~\cite{SpringenbergDBR14} & \textbf{9.08} & \multirow{3}{*}{1.40M} \\ 
       Our I-alternation & 9.12 & \\
       Our I-alternation+MaskT & 13.73 & \\\Xhline{1pt}
    \end{tabular}
}
\label{tb2}
\end{table}

Table~\ref{tb2} conducts an in-depth study of our various baselines without data augmentation. The classification error as well as the number of network parameters are compared. The following observations can be made from the table:
\vspace{-0.1em}
\begin{packed_itemize}

\item Our simple ``I-alternation'' approach already shows successful convergence with competitive performance with respect to the baseline~\cite{SpringenbergDBR14}. Such empirical results are consistent for various network capacities. On the other hand, our approach converges much faster than the SGD baseline (shown later).

\item The ``I-alternation+MaskT'' baseline only applies the binary mask rather than ReLU activations for testing. Not surprisingly the performance drops dramatically, but is not too unacceptable. This suggests the learned mask preserves some discrimination ability for classification, while also confirming the value of standard testing procedure.

\end{packed_itemize}
\vspace{-0.3em}

\begin{table}[!t]
\caption{Comparisons of CIFAR-10 classification error and the number of network parameters.}
\centering
\resizebox{0.85\linewidth}{!}{
\begin{tabular}{!{\vrule width1pt} c !{\vrule width1pt} c|c !{\vrule width1pt}}
    \Xhline{1pt}
       Method & Error (\%) &  \# params \\ \Xhline{1pt}
       \multicolumn{3}{!{\vrule width1pt} c !{\vrule width1pt}}{Without data augmentation} \\ \Xhline{1pt}
       Maxout~\cite{Goodfellow13} & 11.68 & $>6$M \\
       Network in Network~\cite{Lin14} & 10.41 & 1M \\
       Deeply Supervised~\cite{Lee2014} & 9.69 & 1M \\
       ALL-CNN-C~\cite{SpringenbergDBR14} & \textbf{9.08} & 1.3M \\
       ALL-CNN-C (I-alternation) & 9.12 & 1.3M \\ \Xhline{1pt}
       \multicolumn{3}{!{\vrule width1pt} c !{\vrule width1pt}}{With data augmentation} \\ \Xhline{1pt}
       Maxout~\cite{Goodfellow13} & 9.38 & $>6$M \\
       DropConnect~\cite{icml2013_wan13} & 9.32 & - \\
       dasNet~\cite{Stollenga2014} & 9.22 & $>6$M \\
       Network in Network~\cite{Lin14} & 8.81 & 1M \\
       Deeply Supervised~\cite{Lee2014} & 7.97 & 1M \\
       ALL-CNN-C~\cite{SpringenbergDBR14} & 7.25 & 1.3M \\
       ALL-CNN-C (I-alternation) & \textbf{7.20} & 1.3M \\ \Xhline{1pt}
       
    \end{tabular}
}
\label{tb3}
\end{table}

\begin{table}[!t]
\caption{Comparisons of CIFAR-100 classification error.}
\centering
\resizebox{0.75\linewidth}{!}{
\begin{tabular}{!{\vrule width1pt} c !{\vrule width1pt} c !{\vrule width1pt}}
    \Xhline{1pt}
       Method & Error (\%) \\ \Xhline{1pt}
       CNN+tree prior~\cite{Srivastava2013} & 36.85 \\
       Network in Network~\cite{Lin14} & 35.68 \\
       Deeply Supervised~\cite{Lee2014} & 34.57 \\
       Maxout (larger)~\cite{Goodfellow13} & 34.54 \\
       dasNet~\cite{Stollenga2014} & 33.78 \\
       ALL-CNN-C~\cite{SpringenbergDBR14} & 33.71 \\
       Fractional Pooling (1 test)~\cite{Graham14a} & \textbf{31.45} \\
       Fractional Pooling (12 tests)~\cite{Graham14a} & \textbf{26.39} \\
       ALL-CNN-C (I-alternation) & 33.68 \\ \Xhline{1pt}
       
    \end{tabular}
}
\label{tb4}
\end{table}

\begin{table}[!t]
\caption{Comparisons of MNIST classification error.}
\centering
\resizebox{0.75\linewidth}{!}{
\begin{tabular}{!{\vrule width1pt} c !{\vrule width1pt} c !{\vrule width1pt}}
    \Xhline{1pt}
       Method & Error (\%) \\ \Xhline{1pt}
       2-layer CNN + 2-layer NN~\cite{Matthew13} & 0.53 \\
       Stochastic Pooling~\cite{Matthew13} & 0.47 \\
       NIN + Dropout~\cite{Lin14} & 0.47 \\
       ConvMaxout + Dropout~\cite{Goodfellow13} & \textbf{0.45} \\
       NIN + Dropout (I-alternation) & 0.47 \\ \Xhline{1pt}
       
    \end{tabular}
}
\label{tb5}
\end{table}

\begin{table}[!t]
\caption{Comparisons of SVHN classification error.}
\centering
\resizebox{1.0\linewidth}{!}{
\begin{tabular}{!{\vrule width1pt} c !{\vrule width1pt} c !{\vrule width1pt}}
    \Xhline{1pt}
       Method & Error (\%) \\ \Xhline{1pt}
       Stochastic Pooling~\cite{Matthew13} & 2.80 \\
       Rectifier + Dropout~\cite{Srivastava14a} & 2.78 \\
       Rectifier + Dropout + Synthetic Translation~\cite{Srivastava14a} & 2.68 \\
       ConvMaxout + Dropout~\cite{Goodfellow13} & 2.47 \\
       NIN + Dropout~\cite{Lin14} & 2.35 \\
       Multi-digit Number Recognition~\cite{Goodfellow42241} & 2.16 \\
       DropConnect~\cite{icml2013_wan13} & \textbf{1.94} \\
       NIN + Dropout (I-alternation) & 2.36 \\ \Xhline{1pt}
       
    \end{tabular}
}
\label{tb6}
\end{table}

\noindent
{\bf CIFAR-10 and CIFAR-100:} Table~\ref{tb3} and Table~\ref{tb4} further compare our I-alternation approach using the ALL-CNN-C network~\cite{SpringenbergDBR14} with other state-of-the-art methods on CIFAR-10 and CIFAR-100. On CIFAR-10, for cases with and without data augmentation, our method performs on par with or surpasses prior arts. On CIFAR-100, our approach achieves competitive performance as well. Note we use a much smaller ALL-CNN-C network (1.4M parameters) than the top performing Fractional Pooling network~\cite{Graham14a} (50M parameters) thus we train much faster.

\noindent
{\bf MNIST and SVHN:} On the two datasets, we adopt the NIN architecture~\cite{Lin14} with nonlinearities of both ReLU and max pooling. Therefore, the experiments here can be regarded as a generalization test of our I-alternation algorithm to different network architectures and their nonlinearities beyond ReLU. We follow the training hyper-parameters of~\cite{Lin14}, and use the same data preprocessing and splitting for both datasets. Table~\ref{tb5} and Table~\ref{tb6} compare our results with previous methods that do not augment data. Our method performs comparable to the state of the arts again, even though they either use more complicated training schemes or network architectures. This shows the efficacy of our alternation algorithm which is consistent across networks and datasets. We claim that our algorithm can be applied to more complex networks to compare with more recent methods that have even stronger performance on these datasets. The experiments here mainly act as a validation test of our alternation algorithm's efficacy.

\begin{figure*}[t]
\begin{center}
\includegraphics[width=1.0\linewidth]{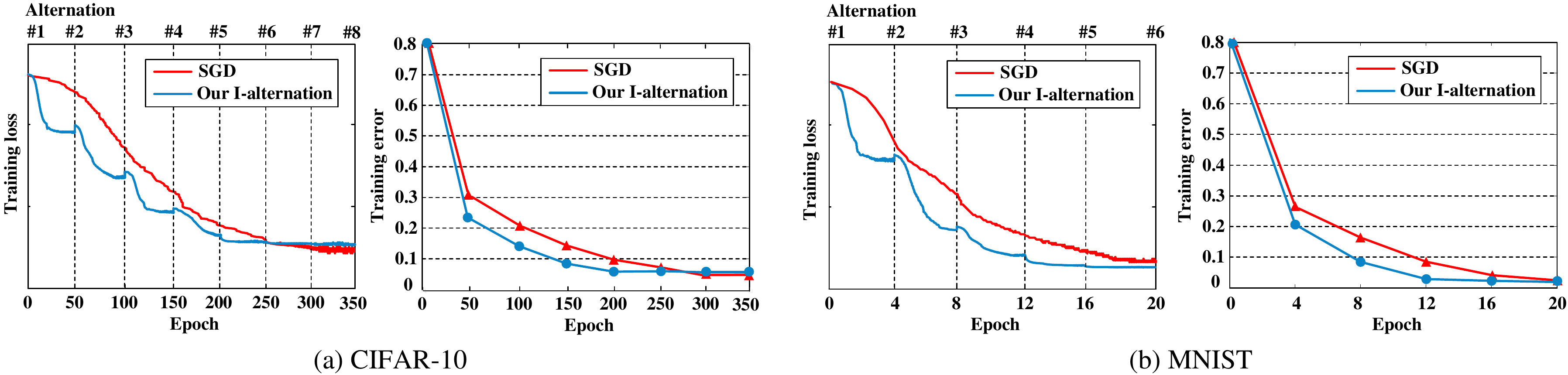}
\end{center}
\caption{Training curves of SGD and our alternation strategy (with instance-wise mask) on CIFAR-10 (without augmentation) and MNIST datasets. Our alternation iteration is identified at the corresponding epoch number. It can be observed that our alternating training reaches similar performance to SGD in 30\%-40\% fewer iterations.}
\label{fig3}
\end{figure*}

\noindent
{\bf Analysis of convergence speed:} One good property of our alternation strategy is that its strong performance comes with a consistent boost in convergence speed. Fig.~\ref{fig3} compares the training curves (in terms of loss and error rate) of traditional SGD and our I-alternation algorithm on two example datasets of CIFAR-10 and MNIST (using different networks).

Note the alternation iteration is identified at the corresponding epoch number, whose gap ranges from 4 to 50 epochs. This gap is approximately determined by the convergence speed of SGD while keeping the mask fixed - when the loss no longer decreases on the considered dataset, we use the passed epochs as the gap and turn to estimate a new mask from the updated weights. After seeing one whole epoch of training examples, we store their individual masks which are then used for the following SGD training. For networks with randomly initialized weights, the generated masks will be meaningless at the beginning of alternation. We run regular SGD for few epochs to obtain reasonable masks to start the whole alternation process. The bias $\beta$ is also initialized this way, and we found good performance can be generally achieved.

It can be observed from Fig.~\ref{fig3} that our I-alternation algorithm converges much faster than conventional SGD. The training loss sometimes increases abruptly at the time of alternation due to the mismatch between outdated mask and updated weights. But overall, our I-alternation algorithm reaches similar performance to SGD in 30\%-40\% fewer iterations. Note when masks are known for all training examples, SGD have been proved to be able to find a global optima under certain conditions, but we can still choose more efficient solvers (\eg~advanced Augmented Lagrange methods or SVD) in the future.

\subsection{Generalization to Large-Scale Dataset}
\label{subsec:imagenet}

We finally discuss computational issues with our I-alternation approach when scaling up to large-scale data (\eg~ImageNet~\cite{NIPS2012_4824}). Specifically, we claim that estimating instance-wise support is storage and computationally demanding for many large scale problems in vision and learning. To this end, we propose a batch-based modification of our ``I-alternation'' algorithm, called ``B-alternation'' where a fixed multi-layer support is shared across a mini-batch of training examples. The closed-form solution to such batch-wise mask can be obtained via solving the problem:
\begin{equation}
    \min_{\m}~\sum_{n=1}^N \Vert \x^{(n)} - \m \odot \x^{(n)} \Vert_2^2 + \beta \Vert \m\Vert_0, ~ m_i \in \{0,1\},
\end{equation}
where $\x^{(n)}$ is one example from the $N$-sized batch. Given $m_i$ as a binary mask in $\m$, the objective function of the corresponding subproblem is
\begin{equation}
    \sum_{n=1}^N (x_i^{(n)}-m_i x_i^{(n)})^2 + \beta m_i = \begin{cases}
        \beta & \text{if} \; m_i=1\\
        \sum_{n=1}^N (x^{(n)})^2 & \text{if} \; m_i=0.
        \end{cases}
\end{equation}

Therefore, the optimal mask solution is:
\begin{equation}
    m_i^* = \begin{cases}
        1 & \text{if}~\sum_{n=1}^N (x^{(n)})^2 > \beta\\
        0 & \text{otherwise}.
        \end{cases}
\end{equation}

In this way, we generate and apply batch-wise masks on-the-fly during training. Table~\ref{tb7} first examines this B-alternation algorithm on small datasets, and finds no accuracy loss compared to the original I-alternation algorithm. The B-alternation algorithm is computationally more friendly as well, which plays a more important role for large datasets. It is worth noting that, we also tried to do the majority vote among individual masks in one batch to produce the batch-wise mask. Although this is not theoretically grounded, we still achieve comparable results to the B-alternation algorithm. We do not report the numbers here to avoid cluttering the experiments.

Next, we apply the novel algorithm to the large ImageNet dataset. We choose the deep GoogLeNet~\cite{Szegedy15}, VGG-16~\cite{Simonyan14c}, ResNet-18~\cite{he2016deep}, which have higher levels of nonlinearities than just ReLU and more practical values due to their wide use. We followed the DSD (Dense-Sparse-Dense)~\cite{HanPNMTECTD17} training flow. We first train a {\em Dense} network to learn full weight connections. In the {\em Sparse} step of~\cite{HanPNMTECTD17}, the network is regularized by pruning the connections with small weights and finetuning the rest. A final {\em Dense} network is obtained by retraining all connections with the previously pruned weights re-initialized from zero. Our B-alternation algorithm is integrated in the {\em Sparse} step due to the sparse nature of the generated masks. By running our alternation algorithm in this step, we simply prune the connections with zero mask, eliminating the need to set the sparsity degree as in~\cite{HanPNMTECTD17}. We train for the same number of epochs, and do not change any other training settings or hyper-parameters.

Table~\ref{tb8} summarizes the Top-1 error rates of the base network, DSD training and its variants using our alternation algorithms. Surprisingly, our method is able to achieve the best results for all types of network architectures on ImageNet. This suggests that our found support offers better guidance than DSD on reducing the weights redundancy. When compared to the original I-alternation algorithm, the B-alternation variant is found to achieve stronger results. One hypothesis for this is that such batch-wise ``regularization'' can help avoid overfitting and escape noisy local minima by instance-wise optimization. Also, the batch-wise operation reflects the current good practices of CNN training and thus enables parallelization in the future. Furthermore, the B-alternation algorithm has typically low computational overhead, while still enjoying the rapid convergence properties of the I-alternation algorithm (about 30\% fewer iterations than SGD).

\begin{table}[t]
\caption{Image classification results of the two variants of our alternation algorithm.}
\centering
\resizebox{1.0\linewidth}{!}{
\begin{tabular}{!{\vrule width1pt} c|c|c|c !{\vrule width1pt}}
    \Xhline{1pt}
     Error (\%) & Baseline & I-alternation & B-alternation \\ \Xhline{1pt}
     MNIST      & 0.47~\cite{Lin14} & 0.47 & \textbf{0.46} \\
     CIFAR-10 (w/o aug)  & \textbf{9.08}~\cite{SpringenbergDBR14} & 9.12 & 9.10\\
     CIFAR-10 (w/ aug) & 7.25~\cite{SpringenbergDBR14} & \textbf{7.20} & 7.22\\ 
     CIFAR-100   & 33.71~\cite{SpringenbergDBR14} & 33.68 & \textbf{33.63}\\ 
     SVHN   & \textbf{2.35}~\cite{Lin14} & 2.36 & 2.38 \\ \Xhline{1pt}
    \end{tabular}
}
\label{tb7}
\end{table}

\begin{table}[t]
\caption{Top-1 error (\%) of the baseline, SDS~\cite{HanPNMTECTD17} and its variants using our alternation algorithms on ImageNet.}
\centering
\resizebox{0.95\linewidth}{!}{
\begin{tabular}{!{\vrule width1pt} c !{\vrule width1pt} c|c|c !{\vrule width1pt}}
    \Xhline{1pt}
     Network & GoogLeNet & VGG-16 & ResNet-18 \\ \Xhline{1pt}
     Baseline & 31.1 & 31.5 & 30.4 \\
     DSD~\cite{HanPNMTECTD17} & 30.0 & 27.2 & 29.2 \\
     DSD (I-alternation) & 30.2 & 26.8 & 29.0 \\ 
     DSD (B-alternation) & \textbf{29.7} & \textbf{26.4} & \textbf{28.7} \\ \Xhline{1pt}
    \end{tabular}
}
\label{tb8}
\end{table}

\section{Conclusion}
This paper reinterprets CNNs as being composed of weights and support using a modified ReLU-nonlinearity. This motivates us to propose a novel alternation strategy between weights and support for CNN training that leads to substantially faster convergence rates and nice theoretical properties. We further prove that, under certain condition, the global optimal solution to CNN weights can be obtained (through local descent) when the multi-layer support is known. Empirical results support the utility and success of our proposed alternation strategy that achieves state of the art results across large scale ImageNet and other standard benchmarks.

As future work, we plan to explore more efficient alternatives than SGD to solve for the weights during our alternation process. This is motivated by the fact that even though we can provide conditions under which local minima are globally optimal through local descent, we still cannot guarantee this in polynomial time and have room for further speedups. We also remark that our theoretical results actually open up an elegant way to connect deep learning with sparse dictionary learning as well. Furthermore, we will show that the proposed alternation strategy can offer guidance on more applications such as network compression and binarization to facilitate efficient network training and storage.

{\small
\bibliographystyle{ieee}
\bibliography{mybib}
}

\end{document}